\documentclass[letterpaper, 10 pt, conference]{ieeeconf}  

\IEEEoverridecommandlockouts                              

\usepackage{hyperref}       
\usepackage{url}            
\usepackage{booktabs}       
\usepackage{amsfonts}       
\usepackage{nicefrac}       
\usepackage{microtype}      
\usepackage{xcolor}         
\usepackage{amsmath}
\usepackage[ruled,linesnumbered]{algorithm2e}
\usepackage[pdftex]{graphicx}
\usepackage{subcaption}
\usepackage{wrapfig}

\newcommand{\vc}[1]{\mathbf{#1}}
\newcommand{\vx}{\vc{x}}

\newcommand{\vu}{\vc{u}}
\newcommand{\cP}{\mathcal{P}}
\newcommand{\cT}{\mathcal{T}}

\newcommand{\todo}[1]{}
\newcommand{\junk}[1]{}
\newcommand{\myvspace}[1]{\vspace*{#1}}

\newtheorem{theorem}{Theorem}
\newtheorem{definition}{Definition}

\title{\LARGE \bf
Comprehensive Reactive Safety: \\
No Need For A Trajectory If You Have A Strategy
}

\author{Fang Da$^{1}$
\thanks{$^{1}$Fang Da is with QCraft Inc,
        {\tt\small fang@qcraft.ai}}%
\thanks{Preprint. Accepted to the 2022 IEEE/RSJ International Conference on
Intelligent Robots and Systems (IROS 2022).}
\thanks{Copyright © 2022 IEEE. Personal use of this material is permitted.
Permission from IEEE must be obtained for all other uses, in any current or
future media, including reprinting/republishing this material for advertising or
promotional purposes, creating new collective works, for resale or
redistribution to servers or lists, or reuse of any copyrighted component of
this work in other works.}
}

\begin{document}

\maketitle
\thispagestyle{empty}
\pagestyle{empty}

\begin{abstract}

Safety guarantees in motion planning for autonomous driving typically involve
certifying the trajectory to be collision-free under any motion of the
uncontrollable participants in the environment, such as the human-driven
vehicles on the road. As a result they usually employ a conservative bound on
the behavior of such participants, such as reachability analysis. We point out
that planning trajectories to rigorously avoid the entirety of the reachable
regions is unnecessary and too restrictive, because observing the
environment in the future will allow us to prune away most of
them; disregarding this ability to react to future updates could prohibit
solutions to scenarios that are easily navigated by human drivers.
We propose to account for the autonomous vehicle's reactions to
future environment changes by a novel safety framework, Comprehensive Reactive
Safety. Validated in simulations in several urban driving scenarios such as
unprotected left turns and lane merging, the resulting planning algorithm called
Reactive ILQR demonstrates strong negotiation capabilities and better safety at
the same time.

\end{abstract}

\section{INTRODUCTION}
\label{sec:intro}


Safety is of central concern in the development of autonomous driving systems,
frequently cited as one of the areas where AI drivers could bring
significant improvement over their human counterparts \cite{av_pub_health, waymo_safety_report2021}.
Indeed, equipped with vastly superior computational power
in solving differential equations and geometric intersections
while never making random mistakes or losing focus, AI vehicle operators
are expected to outperform humans in their ability to foresee and prevent
collisions. This outlook motivated a large body of work employing tools such as
reachability analysis (\cite{14althoff, 20pek, 21li}) to plan or certify safe
trajectories,
on the principle that a trajectory is safe if and only if it can be proven to be
disjoint from any region in space-time that could potentially be occupied by any
(dynamic or static) obstacle. While theoretically appealing, these approaches
prove inadequate when deployed in practical autonomous driving solutions.

One of the major difficulties facing these reachability-based approaches is that
they can be so conservative that it becomes impossible for the autonomous vehicle
to make any progress. Figure \ref{fig:reachability_fail} shows a common scenario
where the autonomous vehicle (abbreviated \emph{AV} below)
is trying to drive past a slow moving vehicle in an adjacent lane. Due to their
potential intention of lane change that is unknowable to the AV, the reachable
set of the slow vehicle will quickly fill up both lanes in a few seconds, which
is not
enough time for the AV to complete the overtake. The range of potential
acceleration on the slow vehicle also allows it to appear in a wide
range of longitudinal positions at any moment, therefore the vehicle's reachable
set essentially covers the entire space-time region in front of the
AV. As a result, in order to avoid any intersection with the reachable set, the
AV has to brake and trail behind the slow vehicle's longitudinal position, as
if following it as a lead vehicle, despite the fact that it is in a different
lane.


\begin{figure}[t]
\centering
\includegraphics[width=0.99\linewidth]{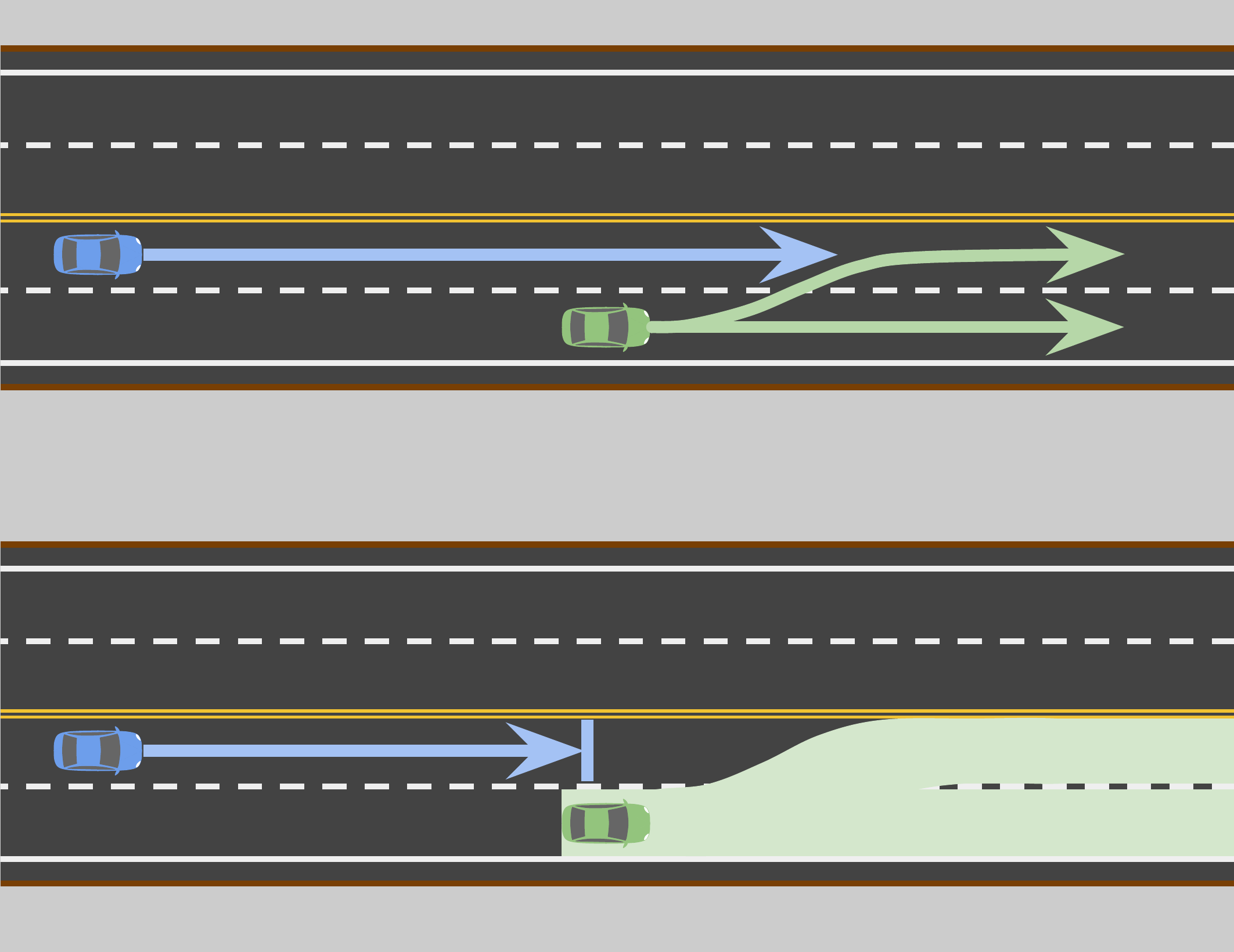}
\caption{A scenario where the fast autonomous vehicle (blue) is overtaking a slow
vehicle (green). The potential lane keeping and lane change trajectories of the
other vehicle are drawn as green arrows on the top; the corresponding (spatial)
reachable sets are drawn on the bottom. The unbounded growth of the
reachable set eventually blocks all the ways to make progress.}
\label{fig:reachability_fail}
\end{figure}

The way current safety techniques utilize
the reachability analysis is conservative by nature. Having the
trajectory being disjoint from the reachable set is well sufficient for proving
the absence of any possibility of collision, but it is far from necessary.
The reachable set of a dynamic object generally grows unboundedly over
time, and given enough planning horizon it will inevitably intersect with any
motion of the AV that makes forward progress. On the other hand, a human driver
readily knows that the slow moving vehicle is physically incapable of fully
blocking both lanes: no
matter which lane it decides to occupy, the other lane will be left open, and%
\footnote{Additionally we have to preclude collisions solely due
to malicious behavior by other vehicles in which the AV is by definition not at
fault, since such collisions are impossible to completely rule out by the AV's
effort alone \cite{rss}. We note that this is not enough
to help a reachability-based AV become unstuck because a potential lane change
early on into the AV's lane is not malicious, yet it blocks the AV's overtake
trajectory and causes braking. See section \ref{sec:overtake} for a concrete
experiment.}
safe overtake is possible. This conclusion is obvious to human drivers,
yet it is unrealized in the autonomous driving safety frameworks available
today. Why?

The key insight here is that in normal driving, the trajectory produced by the
planner is almost never fully executed as is. Modern urban autonomous driving
systems typically run at a frequency around 10Hz, and trajectories (usually with
a horizon of 10 seconds or more) are planned and executed in a pipeline, so
only the initial 100 millisecond portion (about one percent)
of each trajectory is actually
executed by the controller before it receives and switches to execute a new
trajectory. The new trajectory responds to changes in the environment by
incorporating up-to-date sensor data, and thus will in general be
slightly different unless the motion prediction used in the old trajectory is
perfectly deterministic and accurate.
Since by the time the AV gets to the latter portion of the trajectory, it would
have learned more about its surroundings and found smarter things to do,
why be alarmed when the limited knowledge that the AV possesses now cannot
assert safety at that distant time in the future?%
\footnote{See more discussion in Section \ref{sec:rel_work_reachability}.}
This latter portion of the trajectory will not be executed by the vehicle, so
rigorously checking for collisions on this portion of the trajectory does not
benefit the overall safety of the autonomous driving system.
This means that, ensuring the entire
trajectory is free of collision with any potential motion of any obstacle, which
is the paradigm adopted universally by all autonomous driving safety
frameworks known to us, should be reconsidered.
What we need is a new safety framework that
takes into account how the trajectory will change in reaction to environment
changes in the future, and ensures that the \emph{reaction}, instead of the
current \emph{trajectory}, is safe.

To this end, we propose the concept of \emph{Comprehensive Reactive Safety}
(CRS), a condition certifying that the autonomous driving system has the
ability to stay out of collision by reacting to new observations over the
planning horizon, without necessarily requiring it to have a
collision-free trajectory set in stone. This allows the system to respond to
different future scenarios with different tailored actions, as long as it can be
shown that it will have
time to observe and distinguish which scenario is coming to realization before
committing to an action. This freedom is crucial when a single sequence of actions
safely navigating all possible future scenarios does not exist, and thus no
trajectory will pass the traditional trajectory-based safety criterion. In other
words, under CRS, the autonomous driving system is planning strategies%
\footnote{We do not use the term policy here because of a subtle
difference: in reinforcement learning and imitation learning literature,
\emph{policy} is usually used to refer to a function specifying what to do in
each possible state of the system, while we take \emph{strategy} here to refer
to a function specifying what to do in each possible state realizable from the
current state over the planning horizon, \emph{i.e.} a restriction of policy.},
which describe what to do in response to each possible future situation, rather
than planning trajectories, which describe what to do, period.

To summarize, our main contributions are:

\begin{enumerate}
\item We propose the Comprehensive Reactive Safety, a novel safety framework
that, for the first time, accounts for how the AV reacts to future changes in
the environment. This framework drastically increases the flexibility of
planning and opens up possibilities for new solutions to scenarios that are
challenging to traditional reachability analysis. We give proof that the
much less restrictive CRS condition still guarantees safe operation of the
vehicle.

\item We present the Reactive ILQR algorithm, a practical planning algorithm
that satisfies CRS. With side-by-side comparisons in four representative classes
of scenarios commonly found in urban driving, we demonstrate that the proposed
RILQR planner is more effective in negotiating solutions when interacting with
other vehicles, and capable of resolving interactions in ways not possible with
the traditional reachability-based planners, all while improving safety.

\end{enumerate}

The rest of the paper is organized as follows. Section \ref{sec:related_work}
reviews reachability analysis and AV safety approaches. Then, Section
\ref{sec:crs} defines CRS and proves its validity in certifying safe driving,
and Section \ref{sec:rilqr} outlines the RILQR algorithm, which is evaluated in
Section \ref{sec:exps} in common scenarios found on urban public roads.

\section{RELATED WORK}
\label{sec:related_work}

\subsection{Safety enforcement}

Safety is typically defined as the condition and certification that no
collisions exist between the trajectories of the AV and the obstacles,
potentially accounting for uncertainty in both. Planners ensure the safety of
the trajectories they produce by one of two categories of approaches: either
construct the trajectories with safety as a constraint through search
algorithms \cite{11mcnaughton, av_junior} and constrained optimization
algorithms \cite{12xu, av_bertha}, or apply safety verification on the
trajectories constructed with more relaxed
collision constraints, and, upon finding a collision, fall back to the last
verified trajectory \cite{14althoff}, override with a fail-safe trajectory
\cite{20pek}, or perform corrective modifications to restore safety
\cite{17chen}.

Constrained optimization techniques are well studied in the numerical
optimization literature \cite{convex_opt, 14bertsekas}. Motion planning
algorithms for autonomous driving have been developed using various optimization
techniques such as sequential quadratic programming (SQP) \cite{av_bertha},
linear programming (LP) \cite{12xu}, and mixed integer quadratic programming
(MIQP) \cite{16qian_miqp}. Constrained ILQR \cite{cilqr1, cilqr2}, an
extension of the Iterative Linear Quadratic Regulators (ILQR) \cite{ilqr, ilqg}
to support arbitrary state constraints using barrier functions, can be used to
enforce safety constraints as well. We build
our implementation on top of Constrained ILQR.

\subsection{Sources of uncertainty}

Autonomous driving systems operating in the real world are faced with a number
of sources of uncertainty potentially relevant to safety, and they have been the
subject of many studies. Control noise and environment disturbance can be
bounded by feedback control, resulting in the expansion of the AV trajectory
into an invariant set that the controller can guarantee never leaving \cite{15althoff}.
Measurement (in positioning and perception) noise are usually treated by padding
the safety constraints with buffers. Uncertainty in the AV's knowledge about the
environment due to sensor limitations, such as occlusion, can be addressed by
hallucinating obstacles in occluded regions \cite{18orzechowski}. For dynamic
obstacles such as vehicles, their unknown future motions and intentions are a
significant source of uncertainty, for which there are conservative estimation
techniques using reachability analysis \cite{14althoff, 17chen, 20pek}, as well
as a blooming body of research on data-driven motion prediction techniques
\cite{vectornet, tnt, lanegcn} \todo{cite PAGA for camera-ready}.
In particular, Li \emph{et al} \cite{21li} seek to mitigate the conservativeness
of reachability analysis by classifying the object behavior into one
of the clusters learned offline, each of which then has a more limited
reachability. Our work shares their premise, but do not rely on data-driven
procedures that generally lose safety guarantees, and only rule out portions of
reachable sets by evidence from observed object motion.

\subsection{Reachability}
\label{sec:rel_work_reachability}

Reachability analysis is a tool for studying dynamical systems. Some techniques
formulate the process as a differential game and solve its Hamilton-Jacobi PDE
\cite{hamilton_jacobi_reachability}, while others use conservative linearization
to inclusively approximate the nonlinear dynamics \cite{14althoff}. In
autonomous driving applications, reachability analysis is used to compute a
conservative superset of spatial locations an object could reach as a result of
its unknown
intentions and actions, under certain kinematic, dynamic or behavioral
assumptions.

Behavioral assumptions on the objects are adopted to curb the growth of
reachable sets, usually by formalizing traffic laws
\cite{20pek, rss, 12vanholme}; planners avoiding the resulting smaller
reachable sets cannot guarantee to not collide, but can guarantee to not be at
fault in a collision. While it may help to prevent the reachable set
from blocking the AV in some cases, the state of the art in this research
direction has generally been limited to simple rule sets around the
driving direction constraint and the safe distance, struggling to handle the
complexity in real world traffic. The interpretation and formalization of human
traffic rules is a much more difficult task than it might first appear \cite{rss}.
Another way to keep the reachable sets manageable is to limit the planning
horizon as done by Pan \emph{et al} \cite{20pan}, who echo our observation about
the unrealistic spread of reachability in the long term. However, relying on
tuning the horizon to keep reachability under control makes the planner brittle,
and the lack of convergence to the desired behavior as the horizon tends to
infinity is theoretically unsatisfactory. Our
proposed safety framework is not a replacement of reachability analysis, but
rather a more flexible way of applying it, aimed at addressing its weakness of
conservativeness. It can be used with any existing technique to compute
reachable sets.

It is worth pointing out that, although planning around all reachable sets is
unnecessarily conservative because new observations in the future will enable
more targeted responses, this is not true in the worst case
scenario where the autonomous driving system suffers from severe faults and
loses its sensing capabilities after making a plan, which, although rare, is a
risk that cannot be ignored. In such a case, execution of a plan that requires
observation updates will not be possible, and safety can only be achieved by an
open-loop trajectory that circumvents all currently known reachable sets. This
may very well be infeasible, again because of how large the reachable sets could
be, but that is just the harsh reality of trying to stay safe in this extremely
challenging ``glimpse-and-then-go-blind" situation. When designing a safe
autonomous driving system, fallback mechanisms must be included to make the AV
``limp" to a minimum safe state in hazardous situations like this, possibly
involving reachability analysis. However, the necessity of evil in the worst
case is no excuse for not holding ourselves to higher standards in the average
cases, and with our proposed CRS framework, we do.

\section{METHOD}
\label{sec:method}

\subsection{Comprehensive Reactive Safety}
\label{sec:crs}

Consider a given moment in time where the autonomous driving system is supposed
to make a driving plan based on the latest perception and prediction. Without
loss of generality, we call this moment $t=0$. Let $T$ denote the planning
time domain $[0, t_{end}]$ and $X$ the spatial domain (usually the 2D space on
the ground). A trajectory (for the AV or the objects) is a function
$\tau: T \to X$, and the set of all trajectories is denoted by
$\cT = \{\tau | \tau: T \to X\}$. Given a trajectory $\tau$, we use
$C_{\tau}$ to represent the corresponding swept-volume trajectory
$C_{\tau}: T \to \cP(X)$
\footnote{$\cP$ is the power set operator: $\cP(S)=\{s|s \subseteq S\}$.}
such that $C_{\tau}(t)$ is the collision volume of the AV or the object at time
$t$.

Given the set of objects $O$ reported by the perception subsystem, the prediction
subsystem predicts the potential range of motion for each object $i \in O$ as a
set of possible trajectories $P_i \subseteq \cT$. This set
can also be seen as the spatial-temporal reachable set of the object. In
addition, an interaction-aware predictor would be able to give the possible
combinations of individual object motions (taking into account interaction
constraints such as, when two objects approach an intersection from different
directions, one of them has to yield), which we call \emph{futures}
$F = \{(p_1, p_2, ..., p_N) | p_i \in P_i\} \subseteq P_1 \times P_2 \times ... \times P_N$,
a description of all the possible ways the future could unfold. We use the
subscript notation to reference the trajectory of a particular object in a
future: for $f \in F$, $f_i \in P_i$ is the trajectory of object $i$ that will
be realized in this future.

The planning subsystem is responsible for producing a driving plan for the AV
to follow. The traditional planner produces a trajectory
$\tau \in \cT$, and the traditional safety criterion is that it has no collision
with any prediction:
$C_{\tau}(t) \cap C_p(t) = \emptyset, \forall p \in P_i, \forall i \in O, \forall t \in T$.
This condition can be equivalently written in terms of the futures:
$C_{\tau}(t) \cap C_{f_i}(t) = \emptyset, \forall i \in O, \forall f \in F, \forall t \in T$.
On the other hand, in the spirit of ensuring safety by being able to react to
anything that could happen in the future, we propose that the planner
should instead produce a \emph{strategy} which maps potential future situation
realizations to trajectories, $\sigma: F \rightarrow \cT$.
In other words, the planner plans a trajectory for each future scenario.

To ensure the strategy can be followed by the AV (without requiring it to
possess supernatural prophetic powers), it must satisfy certain \emph{causality}
constraints: the trajectory must ``react" to the future variation \emph{after}
it has occurred. For example, an evasive lane change trajectory to circumvent an
object
cutting in must not begin the lane change before the object shows signs that it
is about to cut in.

\begin{definition}[Requirement I: Reaction Causality]
\label{def:reaction_causality}
A strategy $\sigma$ satisfies the Reaction Causality Requirement if for any two
futures $f, g \in F$ and any time $t \in T$, the condition
$\sigma(f)(t) = \sigma(g)(t)$ holds as long as we have
$f_i(t') = g_i(t'), \forall i \in O, \forall t' \in [0, t - \delta]$, where $\delta$
is a sensing delay.
\end{definition}


Here, the sensing delay $\delta$ is the latency between the time a change occurs
in the environment and the time the autonomous driving system finishes
processing the sensor data that pick up the change to be able to recognize it.
Intuitively, Requirement I states that if a strategy responds to two futures
with different trajectories, the trajectories should not differ from each other
at an earlier time than the time when the two futures diverge plus the sensing
delay $\delta$, which is the earliest time possible for the autonomous driving
system to be able to differentiate between the two futures and choose a reaction
accordingly.

\begin{figure}[t]
\centering
\includegraphics[width=0.99\linewidth]{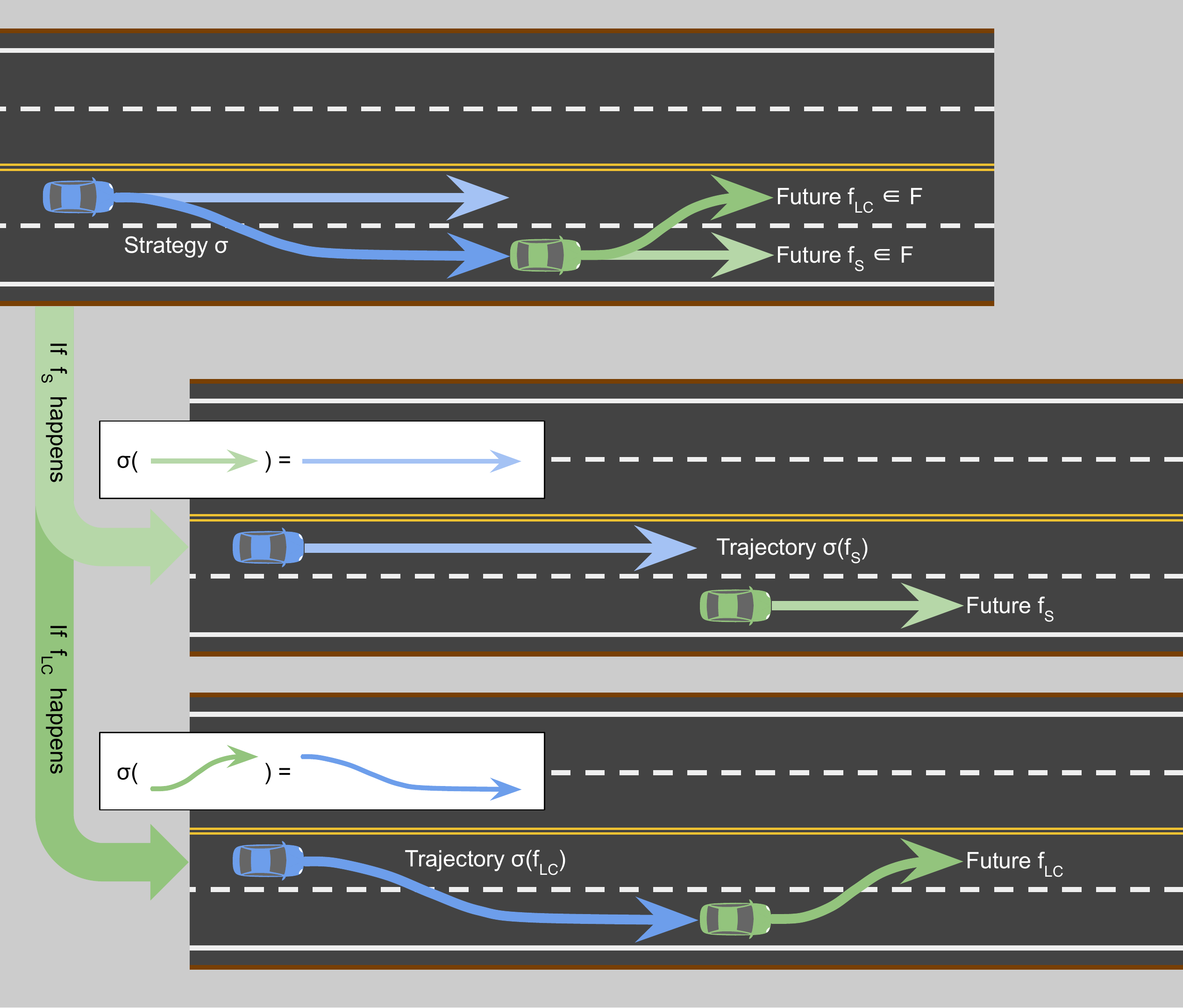}
\caption{A strategy reacting to the two possible futures in the overtake
scenario, labeled LC (for lane change) and S (for straight) respectively.}
\label{fig:strategy}
\myvspace{-2mm}
\end{figure}

The second condition on the strategy is conceptually similar to the traditional
trajectory safety:

\begin{definition}[Requirement II: Reaction Safety]
\label{def:reaction_safety}
A strategy $\sigma$ satisfies the Reaction Safety Requirement if the condition
$C_{\sigma(f)}(t) \cap C_{f_i}(t) = \emptyset$ holds for all $f \in F$, all time
$t \in T$ and all objects $i \in O$.
\end{definition}

We would like to note that Requirement II is significantly weaker than the
traditional safety condition
$C_{\tau}(t) \cap C_{f_i}(t) = \emptyset$,
because the former only requires each trajectory to be free of collision with
its own
corresponding future, rather than with all possible futures. This is the key
distinction that allows the AV to avoid being stuck when the reachable sets
become too spread: each situation can have its own particular solution, and
the solution for one situation does not need to concern itself with handling
other situations. With these two requirements, we are ready to define the
safety condition based on reactions, called Comprehensive Reactive Safety (CRS).

\begin{definition}[Comprehensive Reactive Safety, CRS]
\label{def:crs}
A strategy $\sigma$ is said to satisfy CRS if it satisfies both Requirement I
and Requirement II.
\end{definition}

CRS guarantees safety in the sense of the following theorem, which is our main
result:

\begin{theorem}[Reactive Driving by a CRS Strategy is Safe]
\label{theo:crs}
Consider an autonomous driving system equipped with a strategy $\sigma$
satisfying CRS at time $0$. At any future time $t \in T$, the system would have
observed the motion of each object over time $[0, t - \delta]$; let the motion
of object $i$ so far be denoted by $m_i: [0, t - \delta] \to X$ for each
$i \in O$.
If the system executes trajectory $\sigma(f)$ for any future $f \in F$ that is
compatible with the observed motions so far, that is, $f$ satisfying
$f_i(t') = m_i(t'), \forall t' \in [0, t - \delta], \forall i \in O$, then the
system is
safe from collision with any object.
\end{theorem}

\begin{proof}
Among all the possible futures in $F$, one particular $f^*$ will be the one
eventually occurring in reality. According to Requirement II, if the system
executes $\sigma(f^*)$ over $T$, it is safe from collisions. At any time
$t \in T$, all futures $f$ compatible with observed motions $m_i$ so far,
including $f^*$, are identical up to time $t - \delta$ by definition. According
to Requirement I, their corresponding strategy trajectories $\sigma(f)$,
including $\sigma(f^*)$, are identical up to time $t$. Therefore, executing
$\sigma(f^*)$ is equivalent to executing the $\sigma(f)$ for any compatible $f$
for all $t \in T$, which is what the system does.
\end{proof}

Once the planner successfully plans or verifies a strategy by CRS, it has
a recipe to avoid collisions, even though it does not specify a particular
trajectory for
the controller to follow. As long as the AV keeps observing the environment and
executing the corresponding reactions in a timely manner, it will be safe from
collisions, while enjoying the extra freedom not available in a traditional
trajectory-based framework.

\subsection{Reactive ILQR}
\label{sec:rilqr}

In the previous section we presented a novel safety condition that eliminates
the requirement of conservatively avoiding all reachable sets. However,
implementing such a framework in practice is challenging, due to the transition
from trajectories to strategies. In this section, we present a practical
algorithm in an attempt to leverage the flexibility of CRS in the familiar
Iterative Linear Quadratic Regulator (ILQR) framework. The resulting algorithm,
called \emph{Reactive ILQR}, outputs CRS-worthy strategies, instead of
trajectories. As an early exploration of strategy planners in lieu of trajectory
planners, we focus on the necessary changes in the planner such as the
representation of strategy, and keep the supporting components such as
prediction and decision
as simple as possible, while noting that investigation of those areas
could be interesting research directions.

\subsubsection{Strategy as a tree of trajectories}

Requirement I dictates that all trajectories in a strategy share some prefixes,
some longer than others depending on how early their corresponding futures
diverge. This naturally suggests a representation of strategy as a
tree of trajectories, with the depth direction being time. Two trajectories
sharing a prefix up to time $t$ is stored as a tree with the trunk corresponding
to the common prefix segment over time $[0, t]$, and the two branches to the two
different suffix segments over time $(t, t_{end}]$. Similarly, the futures $F$
are represented as a tree as well.

\subsubsection{Strategy branching according to future branching}

Since the strategy's trajectory $\sigma(f)$ is allowed to be different for each
future $f$,
it follows that the strategy should branch whenever the future branches, or in
other words, the strategy tree should have the same topology as the future tree,
in general. Per Requirement I, the branching point on the strategy tree must
not be earlier than its counterpart on the future tree plus $\delta$, but it can
be later than that, as some situations do not require immediate reactions. For
simplicity, we trivially branches the strategy at time $\delta$ after each
future branching point, in which case situations that do not
require different reactions will simply result in identical strategy branches.

\subsubsection{Strategy optimization by Reactive ILQR}


\begin{wrapfigure}{r}{0.24\textwidth}
\myvspace{-5mm}
\begin{center}
\includegraphics[width=1.0\linewidth]{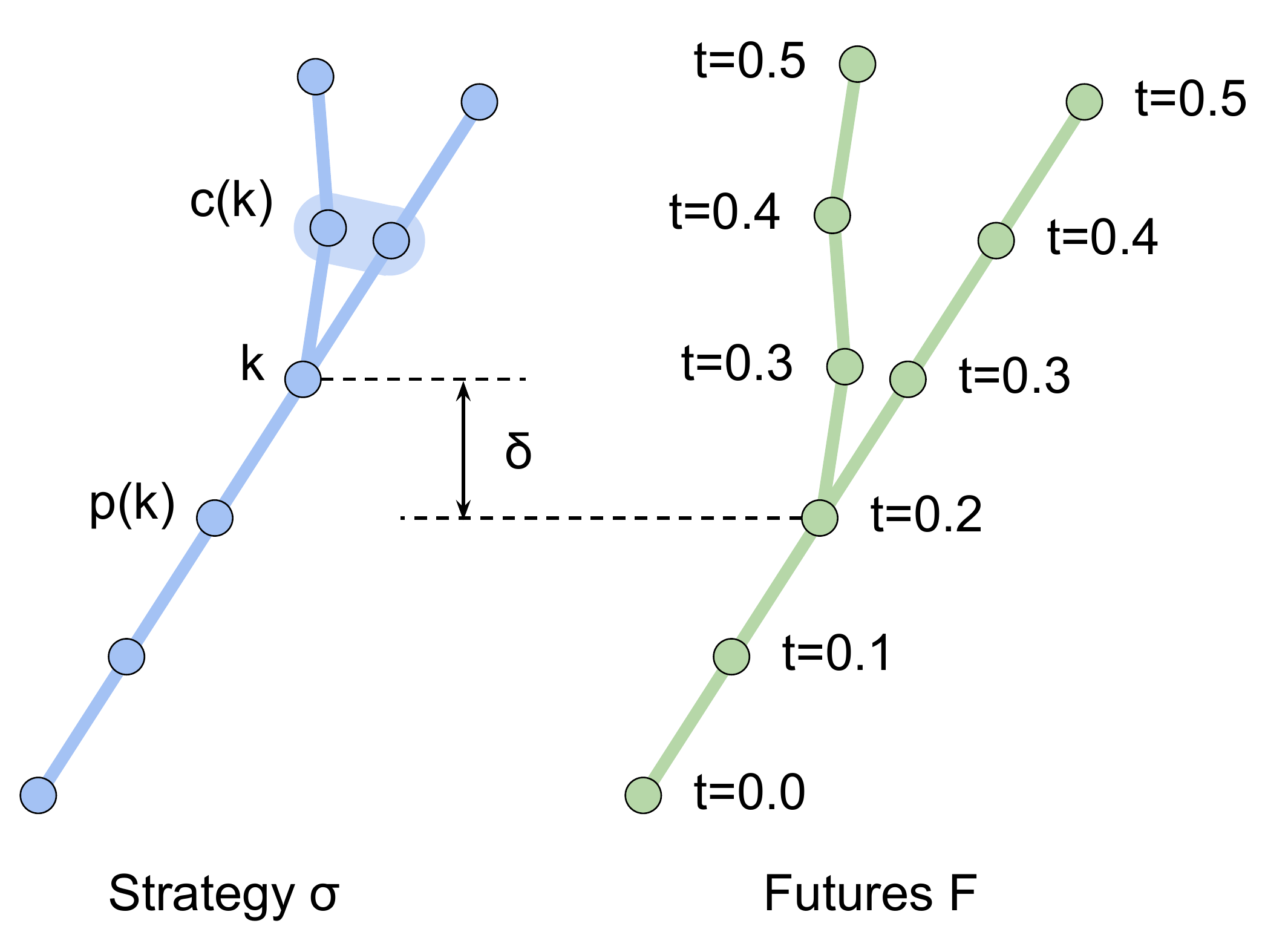}
\end{center}
\myvspace{-2mm}
\end{wrapfigure}

Once the topology of the strategy tree is known, it can be discretized and
variationally optimized, subject to the safety constraints posed by Requirement
II. We temporally discretize the strategy tree at a fixed time step $\Delta t$
(for example, $0.1s$)
into a number of points referenced by a symbol $k$. Unlike the step index in
the original ILQR, $k$ here is an identifier for a point in the strategy tree,
each with a time $t(k)$ that is not necessarily unique, because multiple
branches from the same parent have the same time values.
The root is the only point with time 0: $t(root) = 0$. Each point $k$ (except
for the root) always has one unique parent point $p(k)$ which is
its immediate previous step: $t(k) = t(p(k)) + \Delta t$. Each point $k$ (except
for the leaf points at the end of the horizon) also has one or more child points
$c(k) = \{j|p(j)=k\}$ which are its immediate next steps:
$t(j) = t(k) + \Delta t, \forall j \in c(k)$. A partial order $\prec$ can be
naturally defined for points on a same path from the root; for example,
$k \prec j$ if $k$ is an ancestor of $j$.
Most points have exactly one child though, because branching
occurs sparsely. Branching points segment the tree into a number of linear
chunks, each a trajectory segment containing a number of points, ending with
one having multiple children pointing to the subsequent branch chunks. Extending
the ILQR algorithm to work with such a tree topology leads to a new
algorithm we call \emph{Reactive ILQR}, detailed below.

When considered in isolation, each segment is just a linear trajectory sequence
that ILQR can be applied on straightforwardly. The only tricky part is that the
value function parameters (such as gradients and Hessians) in the backward pass
and the optimal states and controls in the forward pass need to be propagated
across branching points where temporally adjacent trajectory segments meet.
Concretely, given the AV state transition function $g_k$ and the cost function
$l_k$ dictated by the application (potentially varying with $k$), the optimal
control $\vc{u}^*$ (and the associated states $\vc{x}^*$) on a strategy tree
$\sigma$ is given by the following optimization:
\begin{align}
\label{eq:rilqr_opt}
& & \vx^*, \vu^* &= \underset{\vx, \vu}{\arg \min} \sum_{k} w_k l_k(\vx_k, \vu_k),\\
&\mathrm{s.t.} & \vx_k &= g_k(\vx_{p(k)}, \vu_k),\\
& & \vx_{root} &= \vx_{0},\\
\label{eq:h_k}
& & h_k(\vx_k) &\ge 0,
\end{align}
where $\vx_0$ is the current AV state,
$h_k$ is a collision constraint enforcing Requirement II (discussed
later), and $w_k$ is a weight factor normalizing the contribution across sibling
branches with $w_k = \sum_{i \in c(k)} w_i$. One could treat these weights
as probabilities to bias the solution towards certain outcomes (futures);
for simplicity we distribute the weights evenly across all sibling branches at
each branching point by setting $w_k = \frac{1}{|c(p(k))|} w_{p(k)}$, and
$w_{root} = 1$.

The (negative) value function $v_k$ (also known as the cost-to-go function) at
point $k$ on the tree topology is given by
\begin{align}
v_k(\vx_k) = \min_{\{\vu_j | j \succ k\}} \sum_{j \succ k} w_j l_j(\vx_j, \vu_j),
\end{align}
with the following recursion (the \emph{Bellman equation}):
\begin{align}
&v_k(\vx_k) \\
&= \min_{\{\vu_j | j \succ k\}} \sum_{j \in c(k)} \left( w_j l_j(\vx_j, \vu_j) + \sum_{i \succ j} w_i l_i(\vx_i, \vu_i) \right)\\
\label{eq:bellman}
&= \sum_{j \in c(k)} \left( \min_{\vu_j} w_j l_j(g_j(\vx_k, \vu_j), \vu_j) + v_j(g_j(\vx_k, \vu_j)) \right).
\end{align}
ILQR approximates the value function with a quadratic form by its Hessian and
gradient, a form maintained inductively by (\ref{eq:bellman}) provided that
$g_k$ and $l_k$
are in turn approximated linearly and quadratically respectively. This
accomplishes the extension of ILQR to a tree topology. In general, each
branching in the strategy tree doubles the number of variables to be optimized
after the branching point, so the computation cost on a Y-shaped strategy will
be between 1x and 2x of that on a linear trajectory. More branching will further
increase the computation cost, but in the vast majority of urban driving
scenarios, there is no more than one object of immediate concern, so the ILQR
solver for strategy trees is still highly efficient.

The last ingredient in Reactive ILQR is the collision constraints $h_k$ in
(\ref{eq:h_k}). For each point $k$, its constraint $h_k$
only needs to rule out collisions with the futures corresponding to this point:
\begin{align}
h_k(\vx_k) = \min_{f \in \phi(k), i \in O} d(C_{f_i}(t(k)), \vx_k),
\end{align}
where operator $d$ computes the Euclidean distance between the collision volumes
of the AV and an object,
and $\phi(k) = \{f | f \in F, k \in \sigma(f)\}$ is the set of futures
corresponding to point $k$, which can be obtained as the strategy tree is
built according to the structure of the future tree. Once constructed,
the constraints $h_k$ can be enforced in the optimization problem
(\ref{eq:rilqr_opt}-\ref{eq:h_k}) with methods like Constrained ILQR.

\section{EXPERIMENTS}
\label{sec:exps}

The proposed Reactive ILQR algorithm is quantitatively evaluated in a few
representative scenarios. As a baseline, we implement a non-reactive CILQR
planner that computes a single trajectory to avoid all reachable sets. For both
planners, the predicted reachable sets are the same. For simplicity and
comparability, the prediction module is a heuristic kinematic
predictor that outputs an uncertainty box on each time step to represent the
reachable set, and no explicit behavior decision module (making decisions like
``pass" or ``yield") is used. Time discretization is at a resolution of
$\Delta t=0.1s$, and the sensing delay $\delta$ is set to $0.1s$.
The planners are run in a proprietary simulator
with HD maps set up, objects spawned and scripted, and evaluation metrics (such
as collision and interaction outcome) calculated to each scenario specification,
detailed below.

The scenarios described below (Figure \ref{fig:exp_overview}) have infinitely
many variations, parameterized by
manually-designed parameters such as the initial positions of the interacting
vehicles and the actions they take (\emph{e.g.} whether to brake). For each
scenario, we run both RILQR and the baseline CILQR on 100 samples in the
scenario variation space, and compare aggregate statistics to evaluate their
capabilities in staying safe and making progress. For the crossing and
unprotected turn scenarios, the 100-sample batch
experiments are then repeated 10 times to obtain the error bars on the
statistics.

\begin{figure}
\centering
\includegraphics[width=1.0\linewidth]{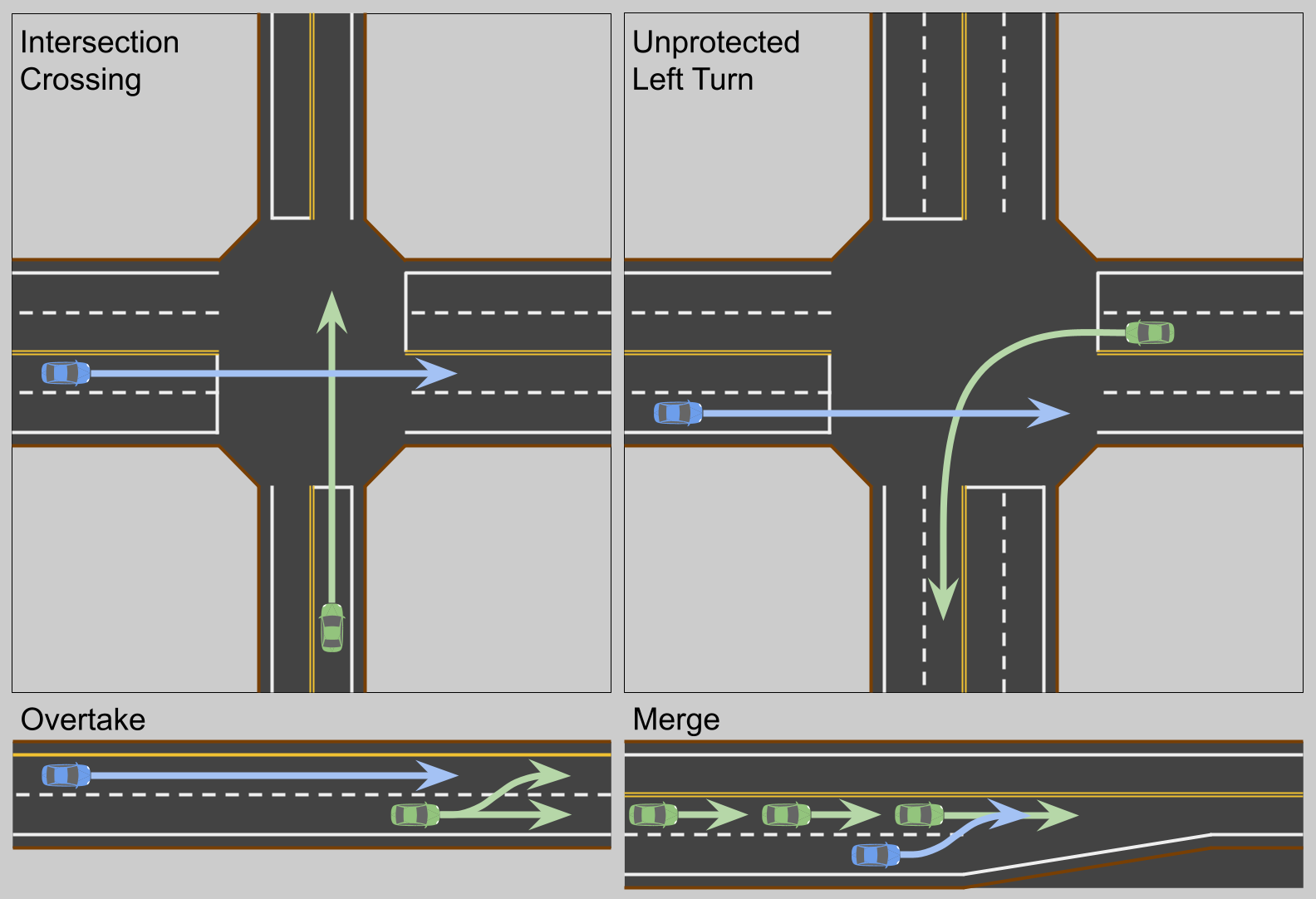}
\caption{The setup of the four representative scenarios. The blue vehicle is the
AV while the green vehicles are the BVs.}
\label{fig:exp_overview}
\myvspace{-3mm}
\end{figure}

\subsection{Intersection crossing}

A common interaction that demands the AV to consider multiple possible actions
of the other vehicle (referred to as the \emph{background vehicle}, \emph{BV})
is the intersection crossing scenario. Both the AV and the BV approach an
uncontrolled intersection from different directions, and at least one of them
must yield to avoid collisions. In our setup the BV is spawned at a location
equidistant to the intersection as the AV, and scripted to either maintain its
speed at $10m/s$ (which is also the AV's initial and target speed) or slow down
before
reaching the intersection by braking at $1.5m/s^2$ for two seconds. Although
the BV's behavior is limited to the binary choice of braking or not braking in
any variation of this scenario, the reachability analysis at any moment cannot
assume the BV will not vary its acceleration, and as a result the reachable
boxes of both choices overlap each other forming a continuous prohibited zone
for a non-reactive planner, making it unable to proceed in either case,
manifested as an excessive inclination to yield. A
reactive planner, on the other hand, would be able to plan a strategy that
yields when the BV does not brake, and passes when the BV brakes. Simulations
sampling the variation of initial longitudinal positions and initial speeds of
the AV and BV confirm this conjecture: while both RILQR and the baseline ensure
that no collisions occur, the AV driven by RILQR is able to pass the BV much
more often (Table \ref{tab:cross_and_upl}). Figure \ref{fig:cross} shows the
behavior of RILQR in a typical variation.

\begin{table}
\centering
\begin{tabular}{lrr}
\toprule
Planner         & Crossing           & Unprotected left turn \\
\midrule
Baseline        & $11.6\% \pm 3.6\%$ & $33.8\% \pm 4.9\%$ \\
\bf{RILQR}      & $35.6\% \pm 5.3\%$ & $57.4\% \pm 5.3\%$ \\
\bottomrule
\end{tabular}
\caption{Passing percentages in the intersection crossing (Fig. \ref{fig:cross})
and unprotected left turn (Fig. \ref{fig:upl}) scenarios. Passing is defined as
moving through the intersection at an earlier time than the other vehicle, and
yielding the opposite. RILQR achieves a higher passing ratio because it does not
conservatively avoid all potential BV actions which is only possible through
excessive yielding.}
\label{tab:cross_and_upl}
\myvspace{-4mm}
\end{table}

\begin{figure}
\centering
\includegraphics[width=0.99\linewidth]{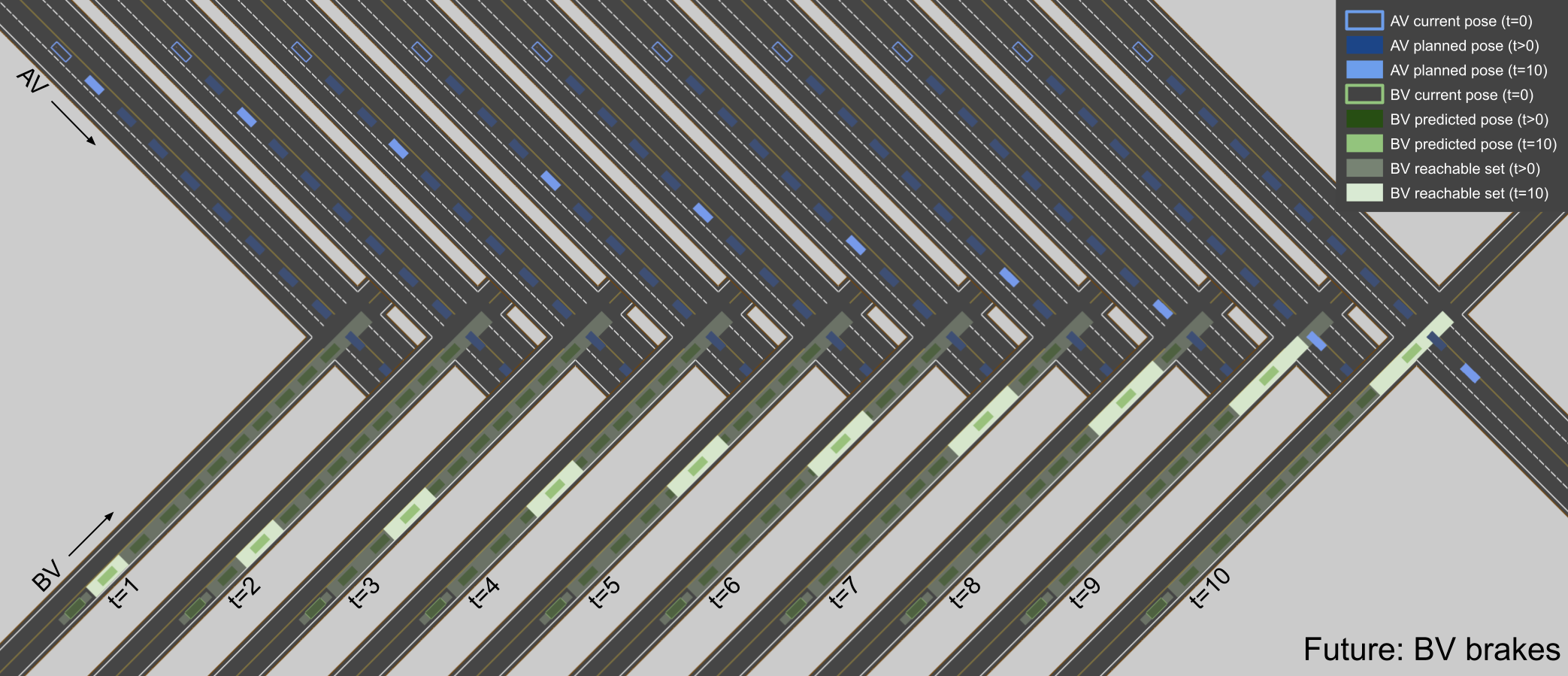}
\includegraphics[width=0.99\linewidth]{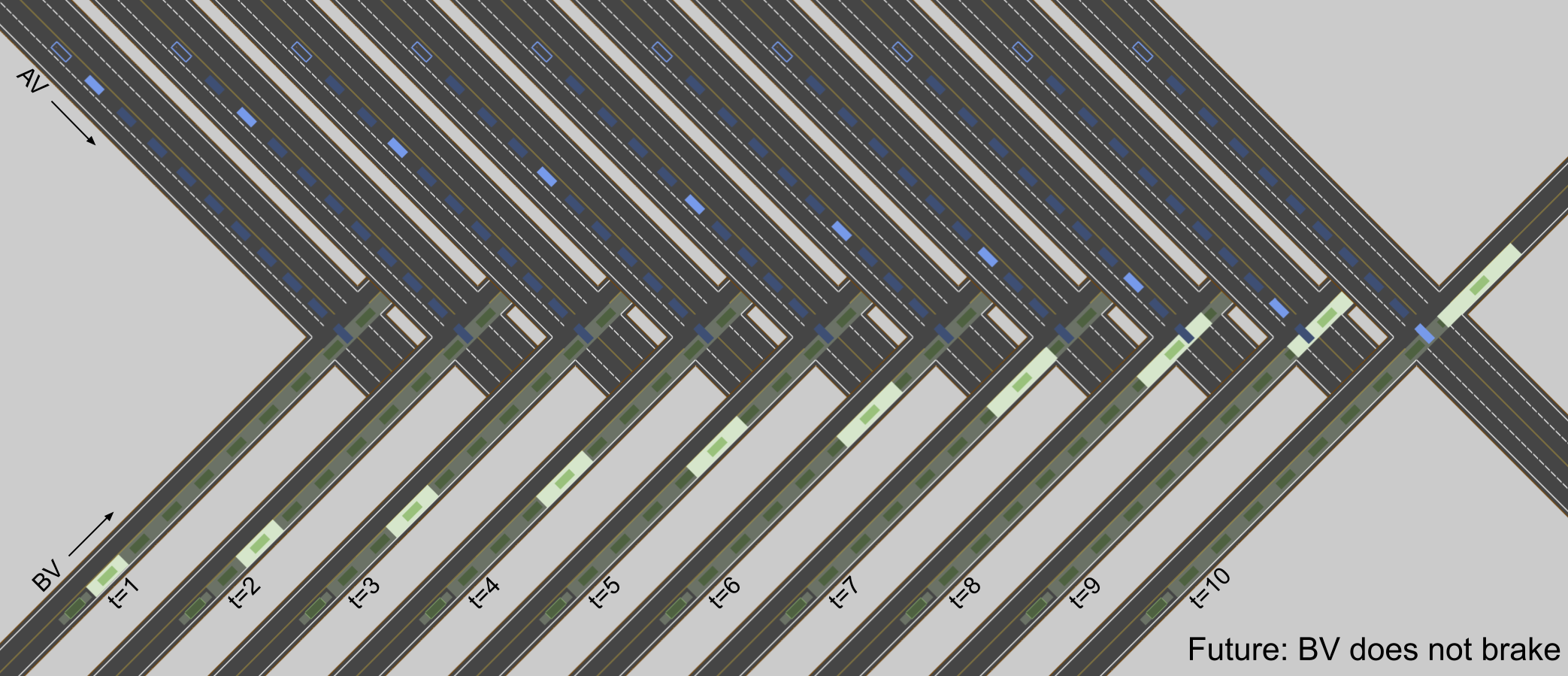}
\caption{RILQR strategy in a crossing scenario variation, at a planning iteration
3 seconds into the simulation. The future where the BV brakes (and thus the AV
proceeds) is shown on the top; the future where the BV does not brake (and thus
the AV yields) is shown on the bottom. Each frame in the horizontal sequence
advances one second into the future, highlighted in bright blue (AV pose in the
strategy) and green (BV pose in the prediction), with the brighter green areas
depicting the reachable sets. Note that all of these are planned at a single
simulation time $t_{sim}=3s$ which is before the potential braking time of the
BV, so the planner cannot know whether the BV ends up braking or not in this
particular scenario variation, nor does it require the knowledge because its
strategy is ready for both.}
\label{fig:cross}
\myvspace{-3mm}
\end{figure}

\subsection{Unprotected left turn}

A similar situation arises when the BV comes from the opposite direction and
intends to make an unprotected left turn. Although the right of way rule
generally requires the vehicle making the
turn to ensure that doing so does not obstruct the surrounding traffic,
violations to this rule occur frequently enough in reality that the AV must be
prepared to yield when the BV recklessly proceeds with the turn. As with the
intersection crossing scenario, RILQR improves the pass ratio without
introducing collisions. Figure \ref{fig:upl} illustrates a typical variation in
this scenario.

\begin{figure}
\centering
\includegraphics[width=1.0\linewidth]{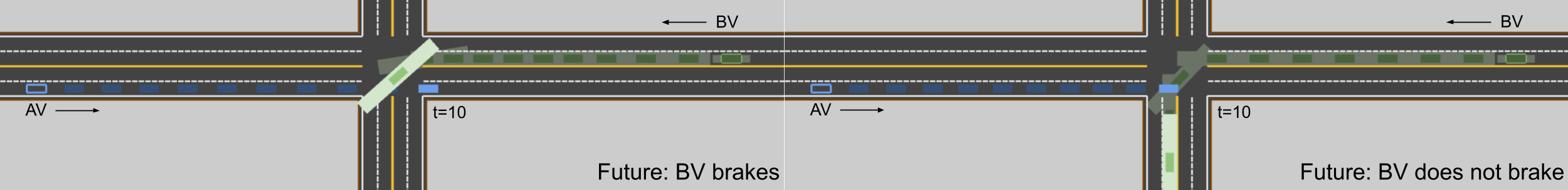}
\caption{RILQR strategy in an unprotected left turn scenario variation, at a
planning iteration 3 seconds into the simulation. The color scheme is the same
as crossing (Fig. \ref{fig:cross}).}
\label{fig:upl}
\end{figure}

\subsection{Overtake on a multi-lane road}
\label{sec:overtake}

As previously discussed, the side-passing scenario (Fig. \ref{fig:reachability_fail})
requires the fast vehicle to understand that although the slow vehicle could
potentially change lanes and thus could appearing in either lane in the
future, an appropriate reaction is always available to complete the overtake.
We design a scenario where the BV, moving at a slow $1m/s$ in the right lane,
may initiate a lane change to the left at a random moment, provided that the
AV is not imminently passing it (in which case the lane change is
considered malicious; see the footnote in Section \ref{sec:intro}). The AV
starts in the left lane, and has a target speed of $10m/s$.
Experiments show that the baseline planner invariably fails to
overtake as expected (see Section \ref{sec:intro}), while RILQR succeeds 100\%
of the time without any collision (Table \ref{tab:overtake}).

\begin{table}
\centering
\begin{tabular}{lrrr}
\toprule
Planner         & Overtake success & BV lane change & AV average speed \\
\midrule
Baseline        & $0\%  $          & $66\%$         & $5.02m/s$ \\
\bf{RILQR}      & $100\%$          & $42\%$         & $9.94m/s$ \\
\bottomrule
\end{tabular}
\caption{Statistics in the overtake scenario (Fig. \ref{fig:overtake}). Overtake
success is defined as the AV being longitudinally ahead at the end of the
simulation. RILQR is able to overtake because it can find a way forward whether
the BV changes lane or not.}
\label{tab:overtake}
\myvspace{-3mm}
\end{table}

\begin{figure}
\includegraphics[width=0.495\linewidth]{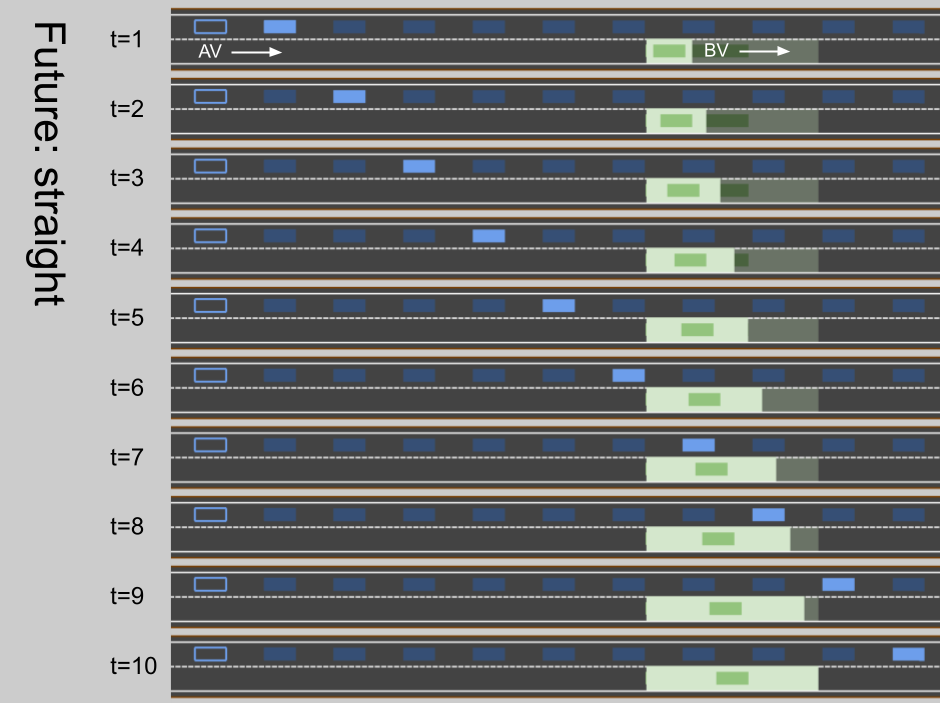}
\includegraphics[width=0.495\linewidth]{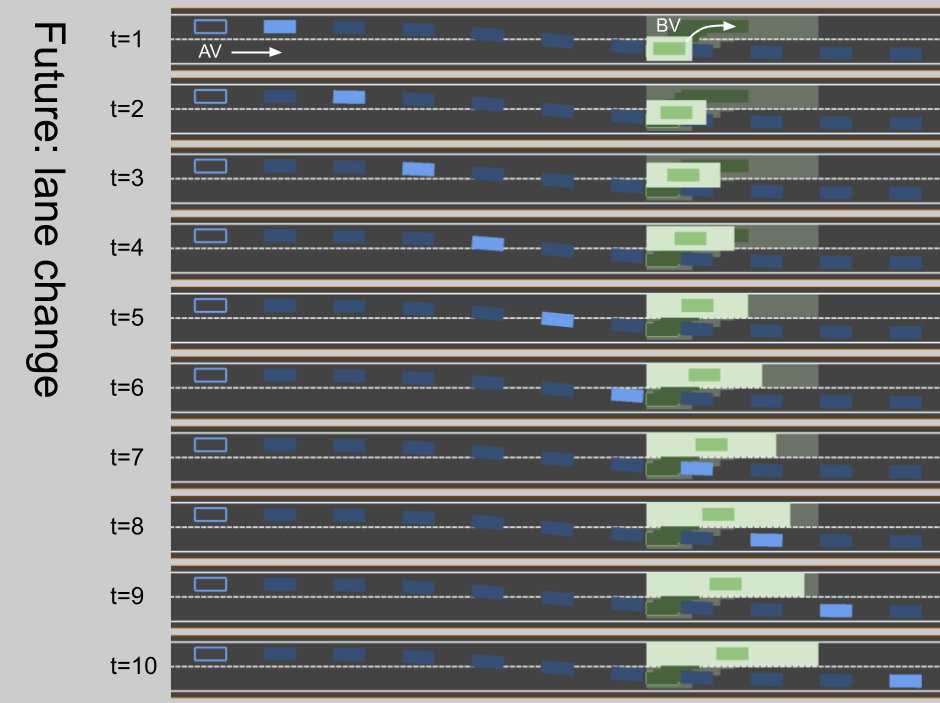}
\includegraphics[width=0.495\linewidth]{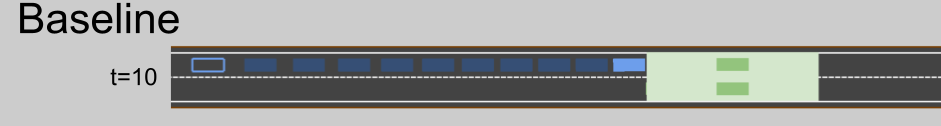}
\caption{RILQR strategy in an overtake scenario variation, at a planning
iteration 4 seconds into the simulation. For comparison, the behavior of the
baseline planner is also included at the bottom.}
\label{fig:overtake}
\myvspace{-3mm}
\end{figure}

\subsection{Merge}

To stress-test RILQR in a complex interaction situation, we set up a merge
challenge scenario where the AV approaches a lane merge into the left lane with
continuous traffic of tightly spaced BVs, each traveling at $10m/s$ (equal to
the AV's default initial and
target speed) and may either brake or maintain speed before arriving at the
merge point, just like
in the crossing scenario. In addition to each BV's individual
braking action, also varied are the AV initial speed and longitudinal position,
to cover the full range of possible initial configurations. The scripted motions
of the BV themselves could sometimes be unrealistic or malicious, especially the
``pincer movement" in certain cases where a BV in the front brakes and another
one right behind it does not brake, and as a result both RILQR and the baseline
encountered collisions, but RILQR does achieve a lower collision rate, and,
more importantly,
a generally more advanced post-merge position in the row of BVs (Table
\ref{tab:merge}), able to complete the merge before the fourth BV passes $88\%$
of the time compared to $30\%$ with the baseline planner, and able to slot in
between the first two BVs twice as often. Since the average
initial position of the AV is longitudinally between the first and the second
BV, this means the behavior of RILQR, instead of being overly aggressive,
is probably closer to a more smooth and fluid merging sequence that avoids the
harsh braking needed to find a merge spot behind latter BVs.

\begin{table}
\centering
\begin{tabular}{lrrrrrr}
\toprule
                &           & \multicolumn{5}{c}{Post-merge position in the BV queue} \\
\cmidrule(r){3-7}
Planner         & Collision & Before 1 & 1-2    & 2-3    & 3-4    & After 4 \\
\midrule
Baseline        & $16\%$    & $1\%$    & $23\%$ & $6\%$  & $0\%$  & $70\%$ \\
\bf{RILQR}      & $5\%$     & $1\%$    & $44\%$ & $31\%$ & $12\%$ & $12\%$ \\
\bottomrule
\end{tabular}
\caption{Statistics in the merge challenge (Fig. \ref{fig:merge}). The numbers
in the post-merge position columns are BV indices. For example, if the AV
merged between the 1st and 2nd BV by the end of the scenario, it is counted as
``1-2", and so on.}
\label{tab:merge}
\end{table}

\begin{figure}
\includegraphics[width=0.495\linewidth]{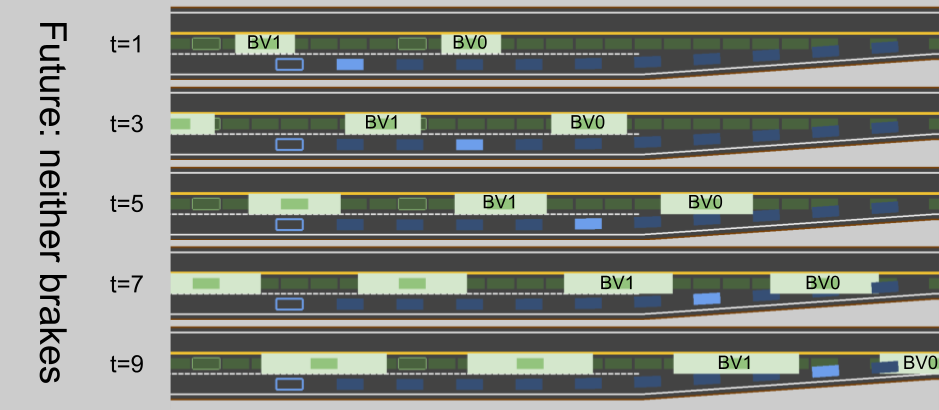}
\includegraphics[width=0.495\linewidth]{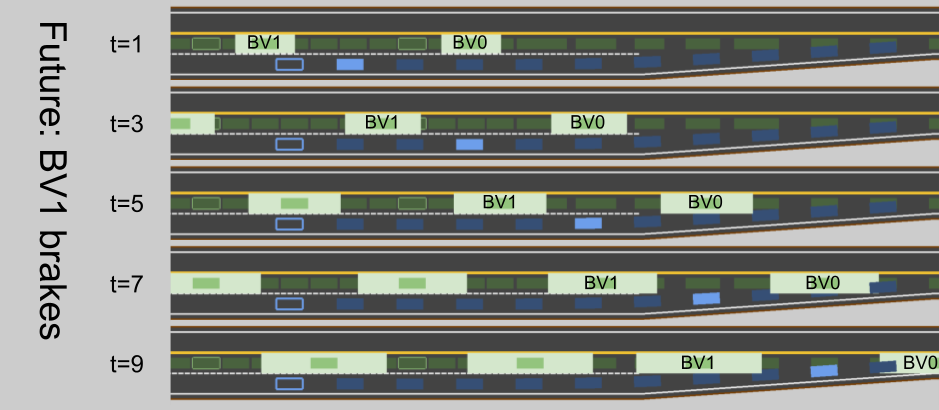}
\includegraphics[width=0.495\linewidth]{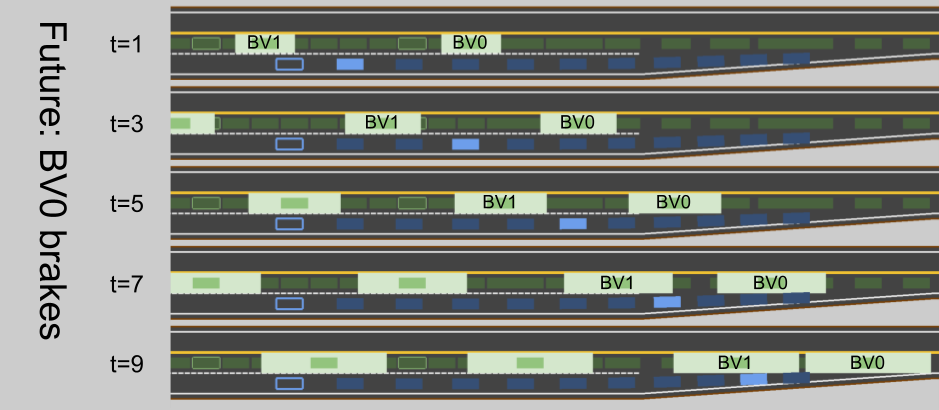}
\includegraphics[width=0.495\linewidth]{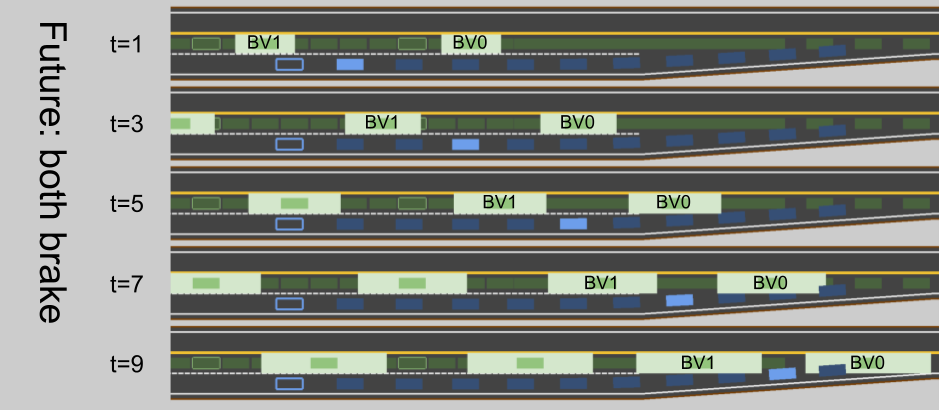}
\caption{RILQR strategy in a merge scenario variation, at a planning
iteration 4 seconds into the simulation. Two objects (BV0 and BV1) branch off
at different times resulting in four distinct futures, and thus four
trajectories in the strategy. The top two trajectories are nearly identical,
both merging between the opening BV0 and BV1. The bottom-right trajectory also
finds the same opening, although overall slower because BV0 brakes. The
bottom-left trajectory cannot merge and has to slow down because BV1's braking
eliminates the opening.}
\label{fig:merge}
\myvspace{-3mm}
\end{figure}


\section{CONCLUSIONS}
\label{sec:conclusions}

In this work, we call for a paradigm shift in the treatment of safety in
motion planning for autonomous driving. Instead of basing the safety condition
on the static knowledge of the currently known environment, we advocate
defining safety with consideration of how the environment will change in the
future, and how the autonomous driving system will react, because that is what
the system does in the general case. This in turn necessitates a shift in the
planner architecture, from planning trajectories to planning strategies. We
outline a safety framework in this context, called CRS, and give an
implementation, called RILQR. The effectiveness of both is demonstrated in
simulations.

This paradigm shift, especially the new \emph{strategy planner} architecture,
opens up many interesting avenues of research. In particular, the module that
partitions the predicted motions of objects from continuous distributions of
trajectories into discrete branches of a strategy plays a crucial
decision-making role because it essentially determines what motions of the
objects need different reactions from the AV, and research on this topic is
scarce. Another example is the question of how to register the current strategy
against the potentially topologically incompatible strategy from the previous
planning iteration, which is often needed to ensure smoothness across plan
handovers in the controller. We believe the concept of strategy planners is
worthy of more investigation effort, and more insight about the planning problem
itself may be uncovered in the process.



%


\bibliographystyle{IEEEtran}
\bibliography{IEEEabrv,bp.bib}

\end{document}